\documentclass[11pt, letterpaper]{article}
\usepackage{fullpage}
\usepackage[T1]{fontenc}
\usepackage{times}
\usepackage{epsfig}
\usepackage{graphicx}
\usepackage{amsmath}
\usepackage{amssymb}
\usepackage{multirow}
\usepackage{subcaption}
\usepackage{xspace}
\usepackage{amsmath,amssymb, amsthm}

\DeclareMathOperator{\argmin}{argmin}
\DeclareMathOperator{\argmax}{argmax}

\newtheorem{definition}{Definition}[section]

\usepackage[pagebackref=true,breaklinks=true,colorlinks,bookmarks=false]{hyperref}

\newcommand{\methodfullname}{{Single-Step Sample Erasure} \xspace}
\newcommand{\methodacronym}{{SSSE}\xspace}

\hypersetup{
    pdftex,
    allcolors=blue,
    bookmarksnumbered=true,     
    bookmarksopen=true,         
    bookmarksopenlevel=1,       
    colorlinks=true,            
    pdfstartview=Fit,           
    pdfpagemode=UseOutlines
}

\newtheorem{prop}{Proposition}

\DeclareMathOperator{\diag}{diag}
\DeclareMathOperator{\rank}{rank}

\begin{document}

\title{SSSE: Efficiently Erasing Samples from Trained \\ Machine Learning Models}

\author{Alexandra Peste\\
IST Austria\\
{\tt\small alexandra.peste@ist.ac.at}
\and
Dan Alistarh\\
IST Austria \\
{\tt\small dan.alistarh@ist.ac.at}
\and 
Christoph H. Lampert \\
IST Austria \\
{\tt\small chl@ist.ac.at}
}
\date{}
\maketitle

\begin{abstract}
   The availability of large amounts of user-provided data has been key to the success of machine learning for many real-world tasks. 
   Recently, an increasing awareness has emerged that users should be given more control about how their data is used. 
    In particular, users should have the right to prohibit the use of their data for training machine learning systems, and to have it \emph{erased} from already trained systems.
   While several \emph{sample erasure} methods have been proposed, all of them have drawbacks which have prevented them from gaining widespread adoption. 
   Most methods are either only applicable to very specific families of models, sacrifice too much of the original model's accuracy, or they have prohibitive memory or computational requirements. 

   In this paper, we propose an efficient and effective algorithm,
   SSSE, for samples erasure that is applicable to a wide class of machine learning models. 
   From a second-order analysis of the model's loss landscape we derive a closed-form update step of the model parameters that only requires access to the data to be erased, not to the original training set. 
   Experiments on three datasets, CelebFaces attributes (CelebA), Animals with Attributes 2 (AwA2) and CIFAR10, show that in certain cases SSSE can erase samples almost as well as the optimal, yet impractical, gold standard of training a new model from scratch with only the permitted data.
\end{abstract}

\section{Introduction}

One of the main reasons for the recent success of deep learning for many computer vision tasks is the availability of large user-provided datasets. 
For example, the popular ImageNet dataset consists of over 14 millions images that were publicly accessible on the Internet~\cite{deng2009imagenet}. 
More recently, Facebook disclosed the existence of an in-house dataset which consists of 3.5 billion Instagram images~\cite{mahajan2018exploring}.

For a long time, when users uploaded their data to online services they silently agreed on transferring a broad range of usage rights to the service. 
Recently, however, this arrangement has come under scrutiny, 
with the growing awareness in our society that user data is valuable, 
that it contains potentially sensitive information, and that online services might use it for purposes the users disagree with. 
Thus, users need to be given more control over how their personal data is used. In recent years, several legal initiatives have been proposed to address this concern. For example, the European Union's General Data Protection Regulation (GDPR) \cite{mantelero2013}  establishes the users' right to withdraw their data and the associated usage rights from the online services at any time. 
While the legal consequences of such a requirement are so far unclear when it comes to machine learning models, it is a realistic possibility that it would imply that the withdrawn data also has to be erased from already-trained models.

Despite the fundamental nature of the problem of erasing certain training examples from an already trained machine learning model, no satisfactory general-purpose solution exists so far. 
There is, of course, one simple solution to the data erasure problem: to simply train a new model on all training data except the withdrawn part.
Unfortunately, this gold standard requires storing and reprocessing all training data whenever a single example should be erased. It is therefore neither practical nor economically viable for real-world models that are often trained over weeks or months from millions or billions of training examples. 
Instead, a number of alternative approaches have been proposed, see Section~\ref{sec:relatedwork} for an overview. 
It is apparent, though, that none of them has found widespread adoption, either in the research community or in commercial applications. 
A possible explanation is that the proposed methods are either not efficient enough to be practical, or not powerful enough to provide satisfactory results.

In this work, we aim at closing that gap, by introducing a new method which we call \emph{\methodfullname (\methodacronym)}. We study the scenario of updating a trained model with parameters $\theta^\star$ on a dataset $\mathcal{D}$, to reflect the removal of a subset of $k$ training samples $\mathcal{S}$. 
By using a first-order Taylor expansion for the gradient of the leave-k-out loss function around the optimizer $\theta^\star$, we obtain a Newton-type update for the model.
This update is inspired from the definition of influence functions \cite{koh2017understanding} and has also been previously studied for convex models in \cite{guo2019certified}.
In comparison to previous work, we address the computational complexity of inverting high-dimensional matrices, by approximating the Hessian with the empirical Fisher Information Matrix (FIM); the latter has the advantage of allowing easy computation with one gradient per data point, and of enabling fast matrix inversion using rank-one updates. 

A similar approach to ours is the ``scrubbing update'' introduced in \cite{golatkar2020eternal}. Differently from \methodacronym, the scrubbing update uses an additional noise term and only the diagonal of the FIM.  
However, we have found through our experiments that the non-diagonal terms of FIM have an important contribution and ignoring them negatively affects the performance of the updated model.

We emphasize that \methodacronym can be used for both convex and non-convex models, as well as for removing a single sample or entire subsets from the training set. Moreover, \methodacronym is broadly applicable to a wide variety of existing machine learning models, as it puts only minor restriction on \emph{how} the original model was trained. We note, however, that the convex setting has better theoretical guarantees, due to the existence of an unique optimum. Additionally, \methodacronym is sensitive to the number of samples being removed, as this affects the quality of the Taylor approximation it relies on. 

Furthermore, we introduce two evaluation methods for sample removal in the (multi-attribute) binary and multinomial classification settings, respectively. Apart from providing additional insight (compared to standard accuracy metrics) into the extent at which samples are removed from a trained model, the proposed evaluation methods also serve as selection criteria for the hyper-parameters of the \methodacronym update. 

In addition to being analytically justified by means of a Taylor expansion of the model's loss landscape, \methodacronym is also deterministic. Consequently, it is easy to apply and understand, including for practitioners, allowing them to, for example, check how far their model of choice fulfills  \methodacronym's underlying assumptions. Moreover, \methodacronym has an important practical advantage, as it only requires access to the data to be deleted, not the original training set, and is also efficient, as samples are erased by simple closed-form updates of the model parameters. 

Before formally defining our proposed \methodacronym method in Section \ref{sec:method}, we first provide an overview of the related work. Subsequently, in Section \ref{sec:experiments} we test the ability of \methodacronym to remove groups of samples from two convex classification models on CelebA \cite{liu2015faceattributes} and Animals with Attributes (AwA2) \cite{xianCVPR17} datasets, as well from a ResNet20 \cite{he2016deep} model trained on CIFAR10 dataset \cite{krizhevsky2009learning}. We also show that the sample removal evaluation methods introduced in Section~\ref{sec:eval-methods} are consistent with more traditional metrics, such as accuracy.

\section{Related Work}\label{sec:relatedwork}
Machine learning traditionally aims at developing methods that absorb a maximal amount of information from a given training set. That a model could forget about parts of the data was mostly seen as a nuisance, e.g. in the context of \emph{catastrophic forgetting}~\cite{mccloskey1989catastrophic,toneva2019empirical}, not as a potentially useful feature. 
Eventually, however, inspired by the concept of  \emph{incremental learning}~\cite{syed1999}, 
where models are trained by sequentially adding 
more and more data, the idea of \emph{decremental learning}~\cite{cauwenberghs01} emerged, in which 
new models are formed by removing samples rather than adding them.  
At that time, the underlying motivation was not data privacy, but a wish to accelerate model selection methods, such as cross-validation, for which models must be trained and evaluated repeatedly on different subsets of the training set~\cite{DudaHartStork01}.
It was shown that decremental learning is possible for a number of classic machine learning models, including \emph{Naive Bayes}~\cite{schelter2020}, \emph{least-squares regression}~\cite{hansen1996}, 
and some variants of \emph{support vector  machines}~\cite{yang2010,tsai2014,karasuyama2010}.
The methods are not \emph{generic} though, but tailored to certain model classes.

More recently, a parallel branch of research emerged under the name \emph{machine unlearning} that concentrated procedures and data structures for accelerating the (re)training of models from subsets of the training data~\cite{bourtoule2021machine,cao15,du2019lifelong,ginart2019,izzo2020, liu2020, wu2020}. 
These efforts were typically limited to specific model classes, or they invoked substantial changes to the original training step, e.g. storing the model's parameter gradients for all training examples. 

The task of manipulating an already trained model to \emph{forget} parts of the training 
data has emerged anew in the context of data privacy 
or \emph{the right to be forgotten}~\cite{villaronga2018}. 
Most existing approaches can be broadly categorized into two groups~\cite{shintre2019}: \emph{statistical} or \emph{analytical}.

Statistical approaches use concepts from probability theory, such as differential privacy~\cite{dwork2014} or vanishing mutual information~\cite{cover2006elements}, to derive criteria and algorithms for successful data erasure.
Both criteria, differential privacy and vanishing mutual information, are quite strict, and thus the resulting methods can provide strong guarantees. 
However, attempts to enforce them typically come with 
a number of drawbacks. In particular, stochasticity 
of the model and the process are unavoidable. 
Consequently, statistical methods cannot be deterministic and typically include a step of parameter randomization, e.g. adding noise~\cite{golatkar2020eternal,golatkar2020forgetting,guo2019certified}. This reduces the  interpretability and reproducibility of results and comes with the danger of drastically reducing the prediction accuracy. 
The need for a non-standard randomized loss function at training time can also be expected to reduce the chance of wide-spread adoption by practitioners.  
Similar problems emerge also for cryptographic protocols, such as~\cite{garg2020}. 

Analytical approaches characterize the removal of data from a trained model by means of properties of the model's loss landscape. 
A prominent example is the use of \emph{influence functions}~\cite{koh2017understanding} to quantify the importance of each training sample to the overall model and to then derive mechanisms for eliminating the influence of the data that should be erased. The analytical approach can lead to methods that are both generic and deterministic and, therefore, \methodacronym also follows this path. 
A major challenge is computational tractability though. Determining the influence of any sample requires computing and inverting the Hessian matrix of the model's loss landscape. Such second-order operations are of prohibitive computational cost for modern high-dimensional models.
Consequently, efficient influence-based data removal has so far only been demonstrated for low-dimensional models~\cite{guo2019certified}. For high-dimensional ones, it is typically discussed only as a theoretical option but then either strong simplifications or more efficient heuristic alternatives are employed.~\cite{izzo2020}.

In this work, we prove this viewpoint to be incorrect. 
We demonstrate that influence-based sample erasure can 
be made tractable even for models that are of realistic 
size and were trained using state-of-the-art 
gradient-based optimization methods. 
Specifically, we circumvent the intractability of 
the inverse of the Hessian matrix by approximating 
it with the inverse of the empirical Fisher information
matrix (FIM). This is possible at the expense of only a single additional scaling parameter,
and subsequently, whenever samples are meant to be 
removed, FIM can be updated efficiently by 
low-rank matrix updates using the Sherman-Morrison 
lemma~\cite{sherman1950adjustment}.
This approach has recently been shown to be effective 
even for large-scale models in the context of 
neural network pruning \cite{singh2020woodfisher}.

Most related to our work on a technical level 
is~\cite{golatkar2020eternal}. The authors also
discuss a second-order update step for sample removal
and suggest using the FIM in place of the Hessian matrix.
Apart from this, the work is quite different from 
ours, though. In particular, the authors introduce 
an additional noise term to ensure that the updated model has a similar statistical behaviour to that 
of one trained from scratch without the removed 
samples. 
Furthermore, due to computational and memory constraints, the authors ultimately use only the noise term to update their model with the additional simplification of approximating the 
FIM by its diagonal. 
In our experience, such an approximation does not capture the real structure of the FIM well, and our experiments in Section~\ref{sec:experiments} confirm these concerns. 

Another important problem related to sample removal methods concerns the lack of consensus regarding the expected behavior of a model after the removal of certain training samples. Most existing methods use standard evaluation metrics, such as the $\ell_2$ distance in the space of parameters, or the accuracy on different data splits. However, these metrics cannot reliably confirm at which extent the targeted samples were removed from the trained model. For example, the $\ell_2$ distance between the updated model and retraining from scratch is not informative in the non-convex case, due to the existence of multiple local minima, and even for convex problems it is not a reliable metric for high dimensional spaces. Furthermore, the accuracy metric might not be able to distinguish between actual sample removal and scenarios when, for example, the connections responsible for predicting a certain class are simply ignored or removed. To address these problems, several research directions have been studied. One of them concerns \emph{membership inference attacks} \cite{shokri2017}, where the goal is to determine whether a given sample was part of the original training set of a model. This is achieved by training multiple ``shadow models'' which mimic the predictions of the original model. More recently, other types of metrics have been proposed, for example the \emph{feature injection test} \cite{izzo2020}, where an extra feature tightly correlated with the target value is introduced, to incur a strong response in the model, before and after the deletion of certain groups of samples. However, this metric has only been used in limited settings, such as linear regression, and it is not clear whether it would result in a meaningful behavior for other models. In this work, we introduce two new evaluation methods for characterizing the removal of samples from a model. The methods we propose look at the similarity between the updated and the retrained from scratch models on the removed samples, and we have found that their behavior correlates well with more traditional evaluation methods, such as accuracy.     

In the following section, we provide a formal 
definition of the problem setting and derive our 
proposed \methodacronym method. 
For the sake of completeness, we also include a
short derivation of the influence-based model 
updates for sample erasure. 

\section{Method}
\label{sec:method}

\subsection{Background}
Consider the setting of supervised learning using a dataset $\mathcal{D} = \{(x_1, y_1), \ldots, (x_n, y_n)\} \subseteq \mathcal{X}\times \mathcal{Y}$.
Let $\mathcal{H}$ be the hypothesis space, containing functions $h_\theta: \mathcal{X} \rightarrow \mathcal{Y}$, twice differentiable in the parameters $\theta \in \mathbb{R}^d$. Consider a twice differentiable loss function $\ell : \mathbb{R}^d \times \mathcal{Y} \rightarrow [0, \infty)$, and for every $i \in \{1, \ldots n\}$, let  $\ell_i(\theta) := \ell(h_\theta(x_i), y_i)$. We incorporate into $\ell$ all regularization terms that are used to optimize the model.
Let $\theta^\star = \argmin_{\theta} L(\theta)$, where
\begin{equation}
    L(\theta) = \frac{1}{n} \sum_{i=1}^n \ell_i(\theta)
\end{equation}

In the following derivation, we assume for simplicity that $L$ is a strictly convex function of $\theta$, such that there exists an unique global minimum $\theta^\star$. We pose the following problem: given a subset $\mathcal{S} \subset \mathcal{D}$, with $|\mathcal{S}| = k$, update $\theta^\star$ in a single step, such that the new model behaves as if it had been trained from scratch on $\mathcal{D}\setminus \mathcal{S}$.

Let $\theta^\star_{\mathcal{-S}} = \argmin_\theta L_{-\mathcal{S}}(\theta)$, where
\begin{equation}
    L_{-\mathcal{S}}(\theta) = \frac{1}{n-k} \sum_{(x_i, y_i) \notin \mathcal{S}} \ell_i(\theta)
\end{equation}
By using a first-order Taylor approximation of $\nabla L_{-\mathcal{S}}(\theta^\star_{-\mathcal{S}})=0$ around $\theta^\star$, we obtain
\begin{equation}
    \nabla L_{-\mathcal{S}}(\theta^\star) + H_{-\mathcal{S}}(\theta^\star)(\theta^\star_{-\mathcal{S}} - \theta^\star) \approx 0
\end{equation}
 
\noindent where $H_{-\mathcal{S}}(\theta^\star) = \nabla^2 L_{-\mathcal{S}}(\theta^\star)$. Since $\theta^\star$ is a minimizer for $L(\theta)$, from the first-order optimality condition we have $\nabla L_{-\mathcal{S}}(\theta^\star) = - \frac{1}{n-k} \sum_{(x_i, y_i) \in \mathcal{S}} \nabla \ell_i(\theta^\star)$. This gives the following approximation for $\theta^\star_{-\mathcal{S}}$:
\begin{equation}
    \theta^\star_{-\mathcal{S}} \approx \theta^\star + \frac{1}{n-k} H^{-1}_{-\mathcal{S}}(\theta^\star) \sum_{(x_i, y_i) \in \mathcal{S}} \nabla \ell_i(\theta^\star)
\label{eqn:hess-update}
\end{equation}

This update step has been discussed previously in the literature as a desirable (but often computationally intractable) approach to sample erasure. 
For example, it appears for this purpose for convex models in~\cite{guo2019certified}, 
where in addition a change to the loss function through random perturbations is proposed. 
The update rule~\eqref{eqn:hess-update} also appears in the context of computing influence functions \cite{cook1980characterizations, koh2017understanding}, except with $H_{\mathcal{D}}(\theta^\star)$, the Hessian over the full training set, used in place of $H_{-\mathcal{S}}$. We refer to the update which uses $H_\mathcal{D}(\theta^\star)$ as the \emph{influence function update}.

\subsection{\methodfullname}

For most practical applications, computing and inverting the Hessian is prohibitively expensive. 
However, when the loss function is the negative log likelihood (i.e. $\ell_i(\theta) = -\log p(y_i|x_i; \theta)$), it is well-known that the expected Hessian over $x \sim p(x)$ and $y\sim p(y|x;\theta)$ denoted by $-\mathbb{E}_{(x, y)}[\nabla^2 \log p(y | x; \theta)]$ is equal to the Fisher Information Matrix (FIM): 
\begin{equation}
    F(\theta) = \mathbb{E}_{(x, y)}[\nabla \log p(y | x; \theta) \cdot \nabla^T \log p(y | x;\theta)]
\end{equation}
\noindent When the discriminative model represented by $p(y |x; \theta)$ is a good approximation for the true conditional distribution $p(y | x)$, we obtain that the Hessian matrix computed over the training set $\mathcal{D}$ can be approximated with an unbiased estimator of the FIM. 
However, the use of FIM requires multiple gradient computations for each data sample, which is impractical for large models and datasets. Instead, a good compromise is the use of the \emph{empirical FIM}, which requires only a single gradient computation per sample, by using the true label of each data point instead of sampling from $p(y|x;\theta)$ multiple times. The empirical FIM is defined as:
\begin{equation}
    \hat{F}_{\mathcal{D}}(\theta) = \frac{1}{n} \sum_{i=1}^n \nabla \log p(y_i | x_i; \theta) \cdot \nabla^T \log p(y_i | x_i; \theta)
\end{equation}

Besides tractability, a major advantage of approximating the Hessian with the empirical FIM is that it allows us to efficiently compute its inverse without having to perform an explicit matrix inversion. This observation has been explored in~\cite{singh2020woodfisher} in the context of pruning neural networks, and we follow here a similar approach.
Namely, we construct $\hat{F}_{\mathcal{D}}^{-1}$ incrementally by a sequence of rank-1 updates, by employing the Sherman-Morrison lemma~\cite{sherman1950adjustment}.
Specifically, we start by choosing $\lambda > 0$ and by defining $\hat{F}_{0}(\theta) = \lambda I_{d}$. 
Afterwards, we incrementally add the contribution of all  $(x_i,y_i)\in\mathcal{D}$.
Namely, by iterating over $i=1,\dots,n$, 
in each step we compute 

\begin{equation}
    \hat{F}^{-1}_{i} = \hat{F}^{-1}_{i-1} - \frac{\hat{F}^{-1}_{i-1} \nabla \ell_{i} \cdot \nabla^T \ell_{i} \hat{F}^{-1}_{i-1}}{n + \nabla^T \ell_{i} \hat{F}^{-1}_{i-1} \nabla \ell_{i},}
    \label{eqn:invFIM}
\end{equation}

\noindent where all gradients and estimates are taken at the fixed $\theta$. 
The final $\hat{F}_n^{-1}$ is identical to the desired  $\hat{F}^{-1}_{\mathcal{D}}$.

The value of the dampening factor $\lambda > 0$ used for initializing the recurrence for the empirical FIM ensures that the matrix is invertible. 
In practice, this is treated as a hyperparameter, but when using $\ell_2$ regularization to train the model, one valid choice for $\lambda$ is the strength of the regularizer. 

Therefore, we propose approximating the Hessian matrix used in the update step from Equation~\eqref{eqn:hess-update} with the empirical FIM. Additionally, we approximate the empirical FIM computed on $\mathcal{D}\setminus \mathcal{S}$ with $\hat{F}_{\mathcal{D}}(\theta)$ and use the latter instead. We noticed, however, that the elements of the empirical FIM usually have a different scale than those of the Hessian. To account for this difference, we add an additional scale hyper-parameter $\epsilon$. In Section~\ref{sec:eval-methods} we describe measures that can be used to select the optimal value of $\epsilon$. Based on these approximations, for any given $\epsilon>0$, we define a \emph{samples erasure} update for a removed subset $\mathcal{S}$ as:
\begin{definition}[\methodfullname Update (\methodacronym)]
    \begin{equation}
    \hat{\theta}_{\epsilon} = \theta^\star + \frac{\epsilon}{n-k} \hat{F}^{-1}_{\mathcal{D}}(\theta^\star)\cdot \sum_{(x_i, y_i)\in \mathcal{S}}\nabla \ell_i(\theta^\star).
    \label{eqn:sample-erasure}
    \end{equation}
\end{definition}

\paragraph{Space and Computational Complexity.} The method proposed above only requires access to the samples that are removed, and not the full train data, as the inverse of the empirical FIM (i.e. $\hat{F}^{-1}_\mathcal{D}(\theta^\star)$) can be computed once and stored for later use. We can additionally make \methodacronym more efficient for large-scale models by assuming a block-diagonal structure of the empirical FIM, or by using mini-batches of gradients (instead of single sample gradients) for the FIM estimation, as demonstrated in~\cite{singh2020woodfisher}.
Specifically, if we employ a block-diagonal structure with blocks of size $B$ and total dimension $d$, and iterate over $n$ samples, the computational complexity of the full iteration is $O(ndB)$, using memory of size $O(dB)$. As shown in our experiments, this allows us to consider fairly large models and sample sets. 

\paragraph{Illustrative Example.} In what follows, we study the decision boundary learned by the \methodacronym update on a simple binary convex classification task, when removing a single sample or a small group of samples. For completeness, we also compare against the influence functions update (see Equation~\ref{eqn:hess-update}), where the Hessian matrix is used instead of FIM. Moreover, for both the influence and \methodacronym updates, we differentiate between using the empirical FIM (or Hessian) computed over all training samples, versus the empirical FIM (or Hessian) on the ``leave-k-out'' (LKO) dataset. As can be seen in Figure~\ref{fig:ex-ssse}, with the right scaling factor $\epsilon$, the decision boundary learned by \methodacronym is similar to the one learned using the Hessian matrix. However, for both methods using the LKO dataset in the FIM or Hessian computation has a significant impact, bringing the decision boundary closer to the one obtained with retraining from scratch.
\begin{figure*}[t]
	\centering
	\begin{subfigure}[t]{0.5\textwidth}
		\centering
		\includegraphics[height=1.4in]{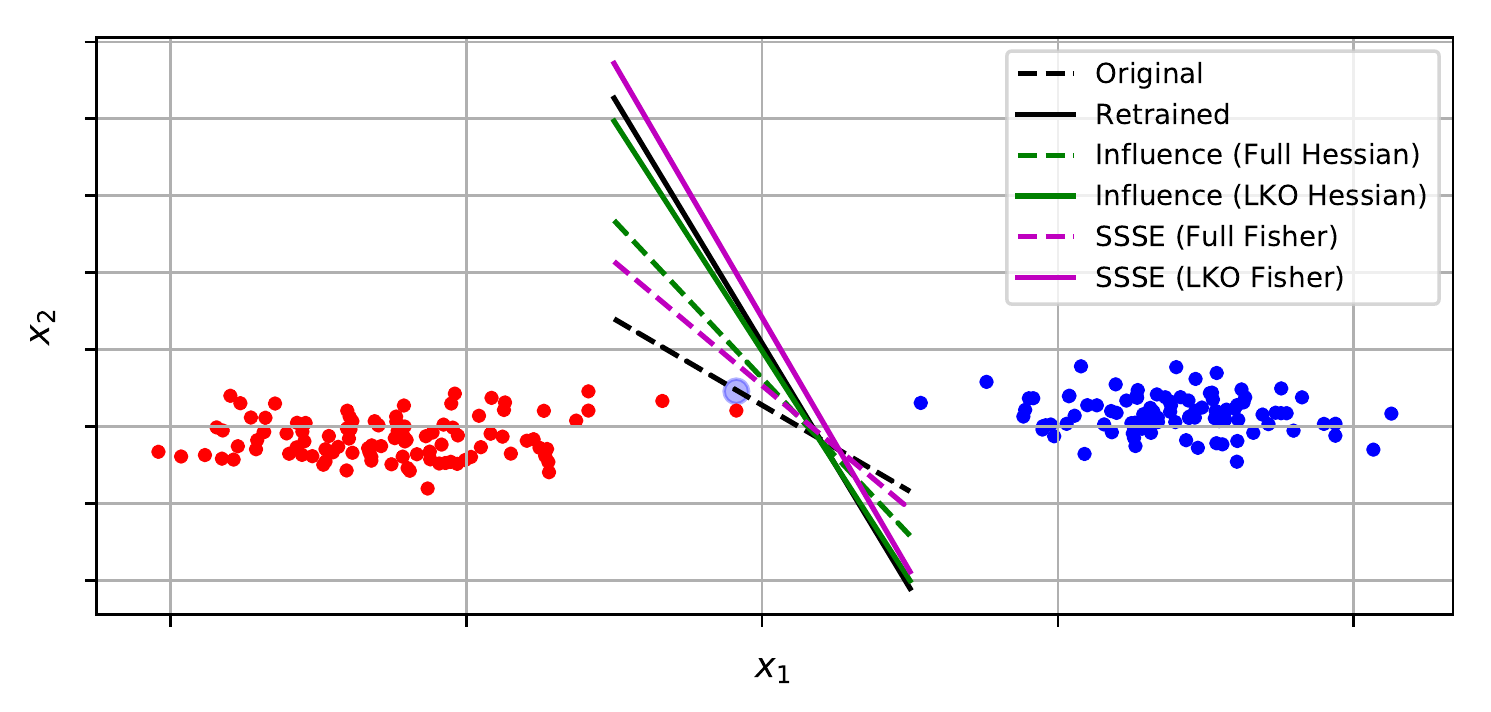}
		\caption{Decision boundary after removing a single sample}
		\label{fig:ex-sample}
	\end{subfigure}%
	~
	\begin{subfigure}[t]{0.5\textwidth}
		\centering
		\includegraphics[height=1.4in]{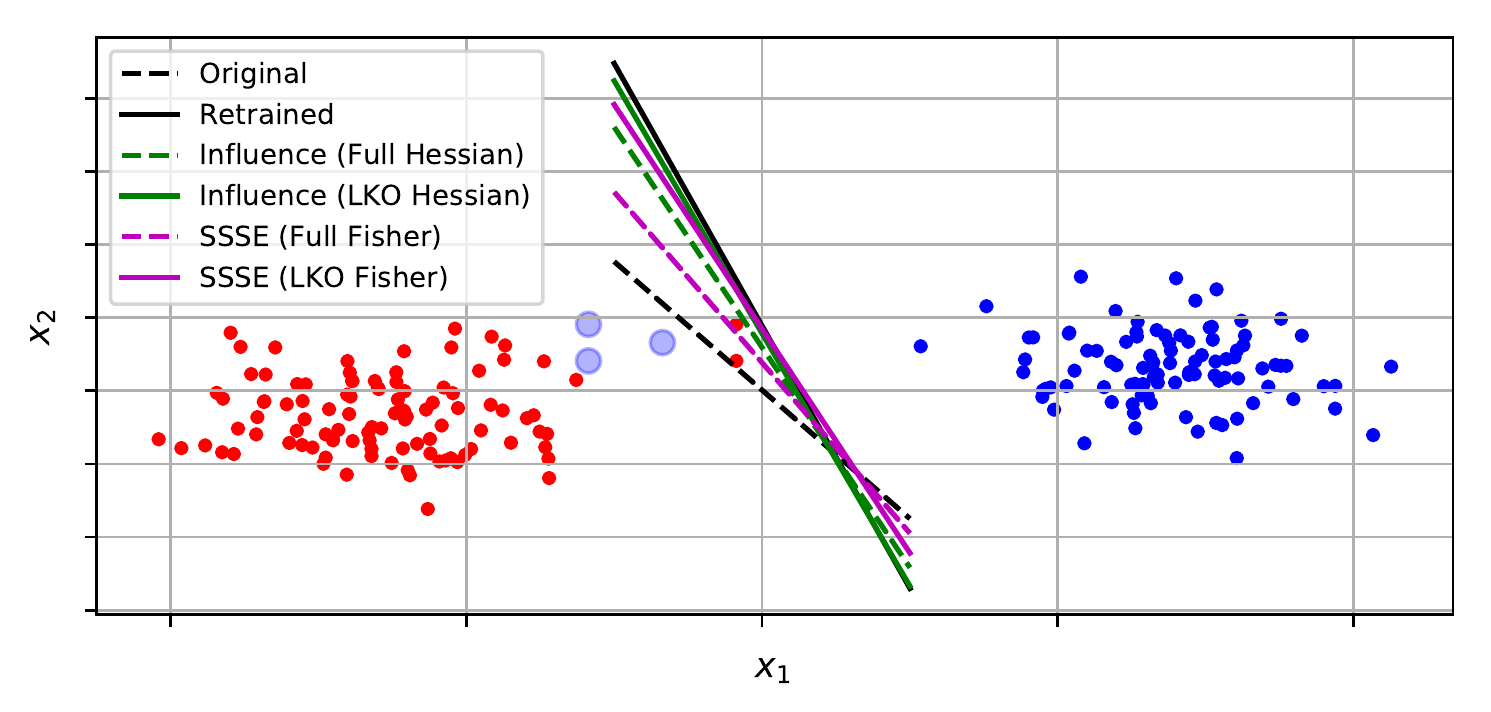}
		\caption{Decision boundary after removing multiple samples}
		\label{fig:ex-multi-sample}
	\end{subfigure}%
	
	\caption{The decision boundary learned with \methodacronym and the influence function update versus the original model and retraining from scratch, when removing a single or multiple samples. We also differentiate between using the Fisher (or Hessian) on the LKO dataset instead of the full training set. The removed samples are represented by the larger transparent blue dots.}
	\label{fig:ex-ssse}
\end{figure*}

\paragraph{Analyzing the Approximation Between Hessian and Empirical FIM.} The \methodacronym update we previously introduced relies on the assumption that the Hessian of the loss function can be approximated using the empirical FIM. In what follows, we will study an example of a multinomial classification problem using linear models, where the Hessian is approximately proportional to the empirical FIM. Namely, we consider a linear model with parameters $\theta$, $h_\theta(x) = \theta \cdot x$, where $\theta\in \mathbb{R}^{c \times m}$, $x\in \mathbb{R}^m$ and $p\in [0, 1]^c$ is the vector of probabilities for each class, obtained after applying a softmax transformation to $h_\theta(x)$.     

\begin{prop}
Let $\theta \in \mathbb{R}^{c\times m}$ be the parameters of a linear model, with $m > c$ and $\rank{\theta}=c$. For any training sample $x$, consider $p\in[0, 1]^c$ the output probabilities per class. Assume that exists $\epsilon\approx 0$, such that for any input label $y\in \{1, 2\ldots c\}$, we have $p_y= 1 - (c-1)\epsilon$ and $p_j = \epsilon$, for $j\neq y$.

Then, if $c\rightarrow \infty$, such that $c\epsilon \approx 0$, we have $H_\mathcal{D}(\theta) \approx \frac{1}{\epsilon(c-1)}\hat{F}_\mathcal{D}$.
\label{prop:hess-fim}
\end{prop}

\begin{proof}{(Sketch)} For simplicity, we consider the Hessian $H_x=\nabla^2 \ell_x(\theta)$ and the empirical FIM $F_x=\nabla \ell_x(\theta) \cdot \nabla^T \ell_x(\theta)$ for a single sample $x$. Via direct calculation, we obtain that:
\begin{align*}
    H_x = (\diag(p) - p p^T) \otimes x x^T \\
    F_x = (p - y){(p-y)}^T \otimes x x^T ,
\end{align*}
where $\otimes$ denotes the Kronecker product. 
Following \cite{singla2019understanding}, the matrix $A = \diag(p) - p p^T \in \mathbb{R}^{c\times c}$ is positive definite (see proof of Theorem 2 from \cite{singla2019understanding}). Assume, without loss of generality, that the one-hot encoding of the label for sample x is $y=(1, 0, \ldots, 0)$. Thus, $p=(1 - (c-1)\epsilon, \epsilon, \ldots, \epsilon)$. 

Using the steps from the proof of Theorem 4 \cite{singla2019understanding}, we can approximate $A \approx c\epsilon u u^T$, where $u\in \mathbb{R}^c$, $u = \frac{1}{\sqrt{c(c-1)}}(1 - c, 1, \ldots 1)$ is the corresponding eigenvector for eigenvalue $c\epsilon$ of A (after ignoring $\epsilon^2$ terms).

Therefore, $H_x = A \otimes x x^T \approx c\epsilon (uu^T) \otimes (xx^T$). Since $p = (1 - (c-1)\epsilon, \epsilon, \ldots, \epsilon)$, it is easy to see that $p - y = \epsilon (1-c, 1, \ldots, 1) = \epsilon \sqrt{c(c-1)} u$, from which $H_x \approx \frac{1}{\epsilon (c-1)} F_x$. The conclusion follows after summing over all training samples.

\end{proof}

\paragraph{Discussion.} We briefly comment on the assumption in Proposition~\ref{prop:hess-fim} regarding the existence of $\epsilon$ such that for all training samples $x$, $p_y = 1 - (c-1)\epsilon$, $p_j=\epsilon$, for $j\neq y$. One example when this assumption would hold is when the data is linearly separable, and all data within a class lies on a $m - c$-dimensional subspace. 

In the case of binary linear classification, a similar result follows trivially, since $H_x = p(1-p) xx^T$ and $F_x = (p-y)^2 xx^T$. Assuming there exists $\epsilon > 0$ such that $|p-y|=\epsilon$ for any sample $(x, y) \in \mathcal{D}$ (i.e. the data lies on two parallel $m-1$--dimensional hyperplanes), we obtain $H = \frac{\epsilon}{1 - \epsilon} F$.
\newline

In the following section, we propose two evaluation methods designed to capture the similarity between \methodacronym and the gold standard of retraining from scratch; moreover, these methods can be used to select the best value of the \methodacronym scale hyper-parameter $\epsilon$. Our experimental analysis on several datasets, on two convex and one non-convex classification tasks, confirms that \methodacronym can be used to erase samples from trained models.

\section{Evaluation Methods}
\label{sec:eval-methods}

In this section we propose several methods for evaluating whether an erasure update indeed fulfilled its purpose of removing samples from a trained model. We focus on convex classification tasks, as the existence of an unique global minimum makes the comparison with retraining from scratch easier. 

We present two evaluation methods, one for the (multi-attribute) binary classification setting, and a second one for multinomial classification. In the following section, we show that these evaluation methods are consistent with the more conventional metrics such as distance in parameter space and accuracy, on the CelebA dataset for multi-attribute binary classification, and on AwA2 for multinomial classification, respectively.

Our motivation for developing these evaluation methods is twofold. First, there is clearly a need in the literature for evaluation methods that can reliably confirm whether a sample is still present in a model. Secondly, we need a method for selecting the best value of the scaling hyper-parameter $\epsilon$ required for the \methodacronym method proposed in Section~\ref{sec:method}. Furthermore, the selection of $\epsilon$ should be done without the risk of over-fitting. 

Intuitively, for any method that aims at removing samples, the largest difference in the behaviour of the updated model is expected to be seen on the set $\mathcal{S}$ of removed samples. For this reason, we define our evaluation methods on the set $\mathcal{S}$; this choice has the additional advantage of reducing the risk of over-fitting when searching for the scaling hyper-parameter $\epsilon$ for \methodacronym.  

\subsection{Multi-Attribute Binary Classification}

A similar behavior between the \methodacronym update $\hat{\theta}_\epsilon$ and the retrained model $\theta^\star_{-\mathcal{S}}$ would be reflected in similar performance on the remaining train set, as well as on the test set, while the removed samples $\mathcal{S}$ should be treated as test samples by both models. In what follows, we develop our evaluation method for the multi-attribute binary classification setting based on this intuition.

Let $A$ be the set of all attributes and let $a \in A$ be a randomly chosen attribute that we wish to (partially) remove from the train set. For the removed train set $\mathcal{S}$, with $|\mathcal{S}| = k$, we refer to $\mathcal{D} \setminus \mathcal{S}$ as the ``leave-k-out train set''. Similarly, if $\mathcal{T}$ is the available test set, we group all samples with a positive attribute $a$ into a subset $\mathcal{T}_a$, called ``removed test set", while $\mathcal{T}_{-a} = \mathcal{T}\setminus \mathcal{T}_{a}$ will be referred to as the ``leave-k-out test set''. We will use the same notations to denote the corresponding data splits in the multi-class setting. Ideally, we want the performances of \methodacronym and retraining to match on all four data splits.

Since we are interested in quantifying how well \methodacronym achieves the goal of removing $\mathcal{S}$, we define a similarity measure between $\theta^\star$ and the retrained model $\theta^\star_{-\mathcal{S}}$, on a particular set of samples. In general, given an attribute $i \in A$ and a model with parameters $\theta$, we define $\alpha^{\mathcal{C}}_i(\theta)$ to be the area-under-the-curve (AUC) score corresponding to the receiver operating characteristic (ROC) curve, computed on a set of samples $\mathcal{C}$. We define the \emph{performance similarity} between two models $\theta_1$ and $\theta_2$ on a set $\mathcal{C}$ as the $\ell_1$ distance between the AUC scores over all attributes:

\begin{definition}[Performance similarity]
    \begin{equation}
    \mathbb{D}_{\mathcal{C}}(\theta_1, \theta_2) = \sum_{i=1}^{|A|} | \alpha^{\mathcal{C}}_i(\theta_1) - \alpha^{\mathcal{C}}_i(\theta_2)|
    \end{equation}
    \label{eqn:auc_l1}
\vspace{-\baselineskip}
\end{definition}

\noindent If an attribute is absent, its corresponding AUC score will be set, by convention, to $0$. We chose the AUC scores over other metrics, such as per-attribute accuracy, since they are the preferred performance measure for imbalanced data, and they are also robust against the class threshold. 
Given Definition ~\eqref{eqn:auc_l1}, we have a quantitative measure of how close the \methodacronym update $\hat{\theta}_{\epsilon}$ is to either $\theta^\star$ or $\theta^\star_{-\mathcal{S}}$ models.

For the available data, the largest gap in terms of AUC scores between $\theta^\star$ and $\theta^\star_{-\mathcal{S}}$ is achieved on the removed samples, since $\mathcal{S}$ is part of the train set for $\theta^\star$, and part of the test set for $\theta^\star_{-\mathcal{S}}$.
Thus, we define the \emph{similarity ratio} between  $\hat{\theta}_{\epsilon}$ and both $\theta^\star$ and $\theta^\star_{-\mathcal{S}}$ on the removed subset $\mathcal{S}$, through the following quantity:

\begin{definition}[Similarity ratio between $\hat{\theta}_\epsilon$ and $\theta^\star$, $\theta^\star_{\mathcal{-S}}$]
\begin{equation}
    \gamma_{\mathcal{S}} (\hat{\theta}_{\epsilon}; \theta^\star, \theta^\star_{-\mathcal{S}}) = \frac{\mathbb{D}_{\mathcal{S}}(\hat{\theta}_\epsilon, \theta^\star)}{\mathbb{D}_{\mathcal{S}}(\hat{\theta}_\epsilon, \theta^\star) + \mathbb{D}_{\mathcal{S}}(\hat{\theta}_\epsilon, \theta^\star_{-\mathcal{S}})}
\label{eqn:auc_gap}
\end{equation}
\end{definition}

Since $\gamma_{\mathcal{S}}(\hat{\theta}_\epsilon; \theta^\star, \theta^\star_{-\mathcal{S}}) \in [0, 1]$, it is clear that $\gamma_{\mathcal{S}}(\hat{\theta}_\epsilon; \theta^\star, \theta^\star_{-\mathcal{S}})=0$ iff $\hat{\theta}_{\epsilon}$ has identical AUC scores to $\theta^\star$ on the set $\mathcal{S}$; similarly, $\gamma_{\mathcal{S}}(\hat{\theta}_\epsilon; \theta^\star, \theta^\star_{-\mathcal{S}})=1$ iff $\hat{\theta}_{\epsilon}$ has the same AUC scores on $\mathcal{S}$ as $\theta^\star_{-\mathcal{S}}$. However, $\gamma_{\mathcal{S}}(\hat{\theta}_\epsilon; \theta^\star, \theta^\star_{-\mathcal{S}})\approx 0.5$ would appear, e.g., when $\theta^\star$ and $\theta^\star_{-\mathcal{S}}$ have an almost identical behavior on set $\mathcal{S}$, or if $\hat{\theta}_{\epsilon}$ diverges from both $\theta^\star$ and $\theta^\star_{-\mathcal{S}}$. Therefore, we define the \emph{maximum similarity between $\hat{\theta}_{\epsilon}$ and $\theta^\star_{\mathcal{-S}}$}
as $\gamma := \max_{\epsilon} \gamma_{\mathcal{S}}(\hat{\theta}_\epsilon; \theta^\star, \theta^\star_{-\mathcal{S}})$ and choose $\epsilon^\star$ as the value achieving this maximum.

In Section~\ref{sec:celeba-results} we show that the value of $\epsilon$ obtained from studying the similarity ratio results in \methodacronym models that are close, in terms of accuracy and loss, to retraining from scratch.

\subsection{Multinomial Classification}

For the multinomial classification setting we follow a different approach towards defining a similarity metric between \methodacronym update $\hat{\theta}_\epsilon$ and retraining from scratch $\theta^\star_{-\mathcal{S}}$. One possibility is to compare the confusion matrices for both models, computed on the set of removed samples. Intuitively, it is easier to perturb a model such that it mis-classifies a certain group of samples (e.g. simply delete the connections to a specific class), than to have it \emph{misclassify in the same way} as a reference model. With this motivation in mind, we define the \emph{confusion distance} between two models $\theta_1$ and $\theta_2$, on a subset of samples $\mathcal{C}$ as:

\begin{definition}[Confusion distance]
    \begin{equation}
        \mathbb{S}_{\mathcal{C}}(\theta_1, \theta_2) = \sum_{i,j} | m^{(i, j)}_{\mathcal{C}}(\theta_1) - m^{(i, j)}_{\mathcal{C}}(\theta_2) |
    \end{equation}
    \label{eqn:confusion_similarity}
\vspace{-\baselineskip}
\end{definition}

\noindent where by $m_{\mathcal{C}}(\theta)$ we denote the confusion matrix of model $\theta$ on samples $\mathcal{C}$. Notice that the confusion similarity is exactly the $\ell_1$ distance between the confusion matrices $m_\mathcal{C}(\theta_1)$ and $m_\mathcal{C}(\theta_2)$, viewed as vectors. 

Next, we define a measure for how close \methodacronym is to $\theta^\star_{-\mathcal{S}}$, compared to $\theta^\star$. Following the same intuition as for the similarity ratio, we define our evaluation metric on the removed samples $\mathcal{S}$. We refer to this metric as the \emph{normalized confusion distance}, which is defined as follows:
\begin{definition}[Normalized confusion distance]
    \begin{equation}
    \delta_{\mathcal{S}}(\hat{\theta}_{\epsilon}; \theta^\star, \theta^\star_{-\mathcal{S}}) = \frac{\mathbb{S}_{\mathcal{S}}(\hat{\theta}_\epsilon, \theta^\star_{-\mathcal{S}})}{\mathbb{S}_{\mathcal{S}}(\hat{\theta}_\epsilon, \theta^\star) + \mathbb{S}_{\mathcal{S}}(\hat{\theta}_\epsilon, \theta^\star_{-\mathcal{S}})}
    \label{eqn:confusion_ratio}
\end{equation}

\end{definition}

\noindent Clearly, if $\delta_\mathcal{S}(\hat{\theta}_{\epsilon}; \theta^\star, \theta^\star_{-\mathcal{S}}) < \frac{1}{2}$, then $\hat{\theta}_\epsilon$ is closer to $\theta^\star_{-\mathcal{S}}$ than to $\theta^\star$, in terms of confusion matrices. Following this definition, we choose $\epsilon$ close to the minimizer of $\delta_{\mathcal{S}}(\hat{\theta}_\epsilon; \theta^\star, \theta^\star_{-\mathcal{S}})$. This ratio only indicates how much closer to $\theta^\star_{-\mathcal{S}}$ is the (mis)-classification behavior of \methodacronym on $\mathcal{S}$, compared to $\theta^\star$. It can also be used as a method for selecting the hyper-parameter $\epsilon$, without the risk of over-fitting, by looking at the performance on the test set. However, similar to the similarity ratio defined previously, the normalized confusion distance does not provide any guarantees in terms of the quality of the resulting model. In Section~\ref{sec:awa2-results} we show that this metric is indeed consistent with other performance measures, such as accuracy, on the AwA2 and CIFAR10 datasets.

\paragraph{Discussion.} We note that both evaluation methods require access to the retrained from scratch model. Naturally, in practice the retrained model is not available when we want to use the evaluation methods to determine the hyper-parameters of \methodacronym. However, we rely on the assumption that the scale hyper-parameter $\epsilon$ is a characteristic of the empirical FIM matrix. Then, we can retrain a single model after removing a set of samples of our choice and use it in the evaluation metrics proposed to select the scale hyper-parameter.

\section{Experiments}\label{sec:experiments}

\subsection{Convex Multi-Attribute Classification}
\label{sec:celeba-exps}

 \begin{figure*}[t]
     \centering
     \hspace{-1cm}
    \begin{subfigure}[t]{0.5\textwidth}
        \includegraphics[height=1.43in]{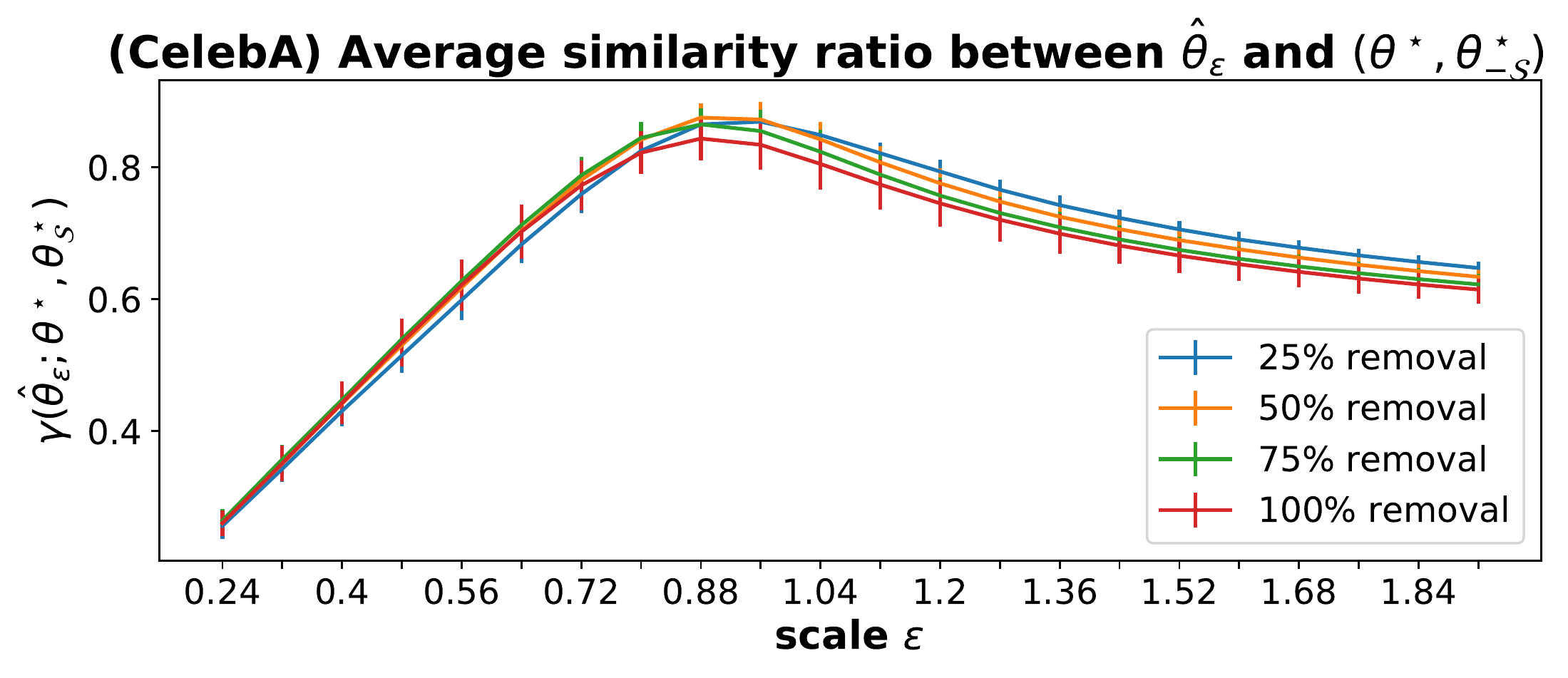}
        \caption{AUC similarity ratio}
        \label{fig:auc-gap-summary}
    \end{subfigure}%
    ~ 
    \begin{subfigure}[t]{0.5\textwidth}
        \includegraphics[height=1.43in]{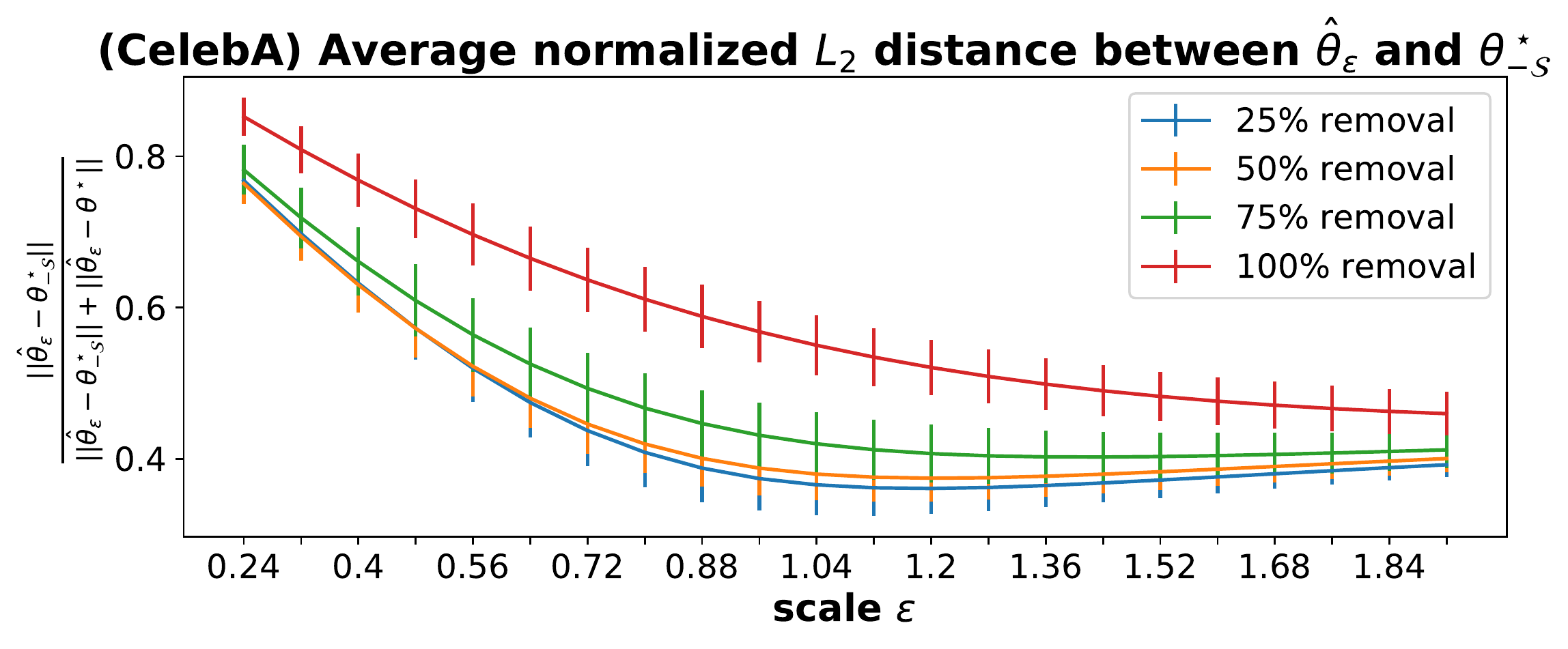}
        \caption{Normalized parameters distance}
        \label{fig:dists_gap_summary}
    \end{subfigure}%
    \label{fig:auc_dists_summary}
    \caption{\textbf{(CelebA)} The similarity ratio $\gamma_{\mathcal{S}} (\hat{\theta}; \theta^\star, \theta^\star_{-\mathcal{S}})$, together with the normalized parameters distance for $\hat{\theta}_\epsilon$ on the removed samples $\mathcal{S}$, as a function of the scaling factor $\epsilon$. The results are averaged across multiple attributes, removed at different rates.}
\end{figure*}
For our first experiment, we examine the task of erasing samples from the large-scale CelebFaces Attributes Dataset (CelebA) \cite{liu2015faceattributes}, where the samples are human faces, each with 40 binary attribute annotations. We use features obtained from a VGG-16 \cite{simonyan2014very} neural network pre-trained on the larger VGGFace dataset \cite{parkhi2015deep}, and randomly subsample $10\%$ of the train set, on which we fine-tune an $\ell_2$-regularized linear multi-attribute classifier to predict each of the 40 binary attributes. The linear model is trained without bias, 
using Stochastic Gradient Descent (SGD) with momentum and a fixed learning rate, and reaches $89.5\%$
average test accuracy across all attributes. We denote by $\theta^\star$ the minimizer of the linear model trained on the randomly sub-sampled train set. 
We perform an extensive analysis to quantify the effect of \methodacronym, by removing different percentages of rare attributes (with at most $20\%$ frequency). 

This particular setting of multiple attributes allows us to better estimate the similarity between \methodacronym and retraining from scratch, by evaluating how the predictions of the model change across the remaining attributes. Although in this case the weights corresponding to each attribute are independent given the train samples, the effect of removing a group of samples with the same chosen attribute has a non-trivial effect on the model, since the data available for the remaining attributes also changes. 

\paragraph{Choice of Hyper-parameters} \methodacronym depends on two sets of hyper-parameters: the \emph{scaling} term $\epsilon$, along with the \emph{dampening} value $\lambda$, used to ensure the empirical FIM is invertible. We fix $\lambda$ equal to the $\ell_2$ regularization coefficient; in general, we noticed that the value of $\lambda$ can influence the performance of the method. To choose the value of $\epsilon$, we perform an extensive hyper-parameter grid search and use the similarity ratio introduced in Definition~\ref{eqn:auc_gap}. We found, however, that the optimal value is fairly consistent throughout different attributes and removal rates, which suggests that $\epsilon$ is in fact a property of the empirical FIM, and is less connected to the gradients of the removed samples. We choose the scaling factor $\epsilon$ as the optimal value which maximizes the similarity ratio.  

\subsubsection{Results}
\label{sec:celeba-results}

\begin{table*}[]
\centering
\scalebox{0.8}{
\begin{tabular}{|l|l|l|l|l|l|l|l|l|l|l|l|l|l|}
\hline
Attr. & Model & \multicolumn{2}{l|}{$L_{-\mathcal{S}}(\theta)$} & 
\multicolumn{2}{l|}{$\| \nabla L_{-\mathcal{S}}(\theta)\|$} & 
\multicolumn{2}{l|}{$\mathcal{D} \setminus \mathcal{S}$ acc (\%)} & \multicolumn{2}{l|}{$\mathcal{S}$ acc (\%)} & \multicolumn{2}{l|}{$\mathcal{T}\setminus \mathcal{T}_a$ acc (\%)} & \multicolumn{2}{l|}{$\mathcal{T}_a$ acc.(\%)} \\ \hline \hline
\multicolumn{1}{|c|}{
\multirow{4}{*}{
\begin{tabular}[c]
{@{}c@{}}15
\end{tabular}}} 
&  & 100\%  & 50\%  & 100\% & 50\%  & 100\% & 50\% & 100\% & 50\%  & 100\%  & 50\%  & 100\% & 50\%  \\ \cline{2-14} \cline{2-14} 
\multicolumn{1}{|c|}{}
& $\hat{\theta}_{\epsilon}$  &  0.235  & 0.235 & 0.004 &  0.004 & 90.04  & 90.02 & \textbf{87.94}   & \textbf{89.08}  & 89.49 &  89.5  & \textbf{87.68} & \textbf{88.69} \\ \cline{2-14} 
\multicolumn{1}{|c|}{}  
& $\theta^\star_{\mathcal{-S}}$ &  0.233 & 0.234  & 0.003 & 0.003  & 90.07    & 90.05   & \textbf{87.2}   & \textbf{89.16} & 89.51  & 89.49 & \textbf{86.92} &  \textbf{88.66}       \\ \cline{2-14} 

\multicolumn{1}{|c|}{}
& $\theta^\star$ & 0.234 & 0.234 & 0.004  & 0.004 & 90.05 & 90.05 & 90.05 &  90.11 & 89.51 & 89.51 &  89.48 & 89.48 \\ \hline \hline
\multirow{4}{*}{
\begin{tabular}[c]
{@{}l@{}}35
\end{tabular}}
&  & 100\%  & 50\% & 100\% & 50\%  & 100\%  & 50\% & 100\%  & 50\%&    100\%  & 50\% & 100\%  & 50\%  \\ \cline{2-14} 
& $\hat{\theta}_{\epsilon}$  & 0.235 & 0.235 & 0.004 & 0.004 & 90.0  & 90.01 & \textbf{89.47} & \textbf{90.48} & 89.46  & 89.47 & \textbf{88.98} & \textbf{89.85} \\ \cline{2-14} 
& $\theta^\star_{-\mathcal{S}}$ & 0.234  & 0.235 & 0.003 & 0.003 & 90.03 & 90.03 & \textbf{88.86} & \textbf{90.24}  & 89.47 & 89.48 & \textbf{88.44}  & \textbf{89.7} \\ \cline{2-14} 
& $\theta^\star$ & 0.235 & 0.235 & 0.004 & 0.004 & 90.0  & 90.02 & 91.16 & 91.31 & 89.47 & 89.47 & 90.37  & 90.37  \\ \hline
\end{tabular}}
\caption{\textbf{(CelebA)} Results for the \methodacronym update $\hat{\theta}_{\epsilon}$, along with the performances of $\theta^\star$ and $\theta^\star_{-\mathcal{S}}$, on two randomly chosen rare attributes: 15 (``Eyeglasses'') and 35 (``Wearing Hat''), removed fully or partially (50\%). We report the average accuracies over all attributes, for each of the four dataset splits, along with the average loss over attributes and the gradient norm on $\mathcal{D}\setminus \mathcal{S}$. For $\epsilon$ we chose $\argmax_{\epsilon}\gamma_\mathcal{S}(\hat{\theta}_\epsilon; \theta^\star, \theta^\star_{-\mathcal{S}})$.}
\label{table:results_celeba}
\end{table*}

We compare the \methodacronym update against the gold standard of retraining from scratch without samples $\mathcal{S}$. Given the converged model $\theta^\star$ on $\mathcal{D}$, we first compute and store $\hat{F}^{-1}_\mathcal{D}(\theta^\star)$, and use it later on to compute each sample removal update. We fix the dampening $\lambda=10^{-4}$, which is equal to the $\ell_2$ regularization parameter. 
Due to the large number of parameters of the model, the computation of $\hat{F}^{-1}_\mathcal{D}(\theta^\star)$ is expensive. However, in this particular case, the matrix has a block-diagonal structure, with 40 equal blocks of size 4096, corresponding to the individual weights of each attribute. 

Once $\hat{F}^{-1}_\mathcal{D}(\theta^\star)$ is computed, we simulate the removal of different ratios of a single rare attribute by computing the \methodacronym update, while also retraining to obtain each $\theta^\star_{-\mathcal{S}}$. This procedure is repeated for all attributes that appear in at most $20\%$ of the train samples, which consist of almost half of all available attributes. We search the best value of the scale $\epsilon$ for \methodacronym using $\gamma_\mathcal{S}(\hat{\theta}_\epsilon; \theta^\star, \theta^\star_{-\mathcal{S}})$ (defined in Equation~\ref{eqn:auc_gap}) as our performance measure.

The results in Figure~\ref{fig:auc-gap-summary} show the similarity ratio $\gamma_{\mathcal{S}}(\hat{\theta}_\epsilon; \theta^\star, \theta^\star_{-\mathcal{S}})$ averaged over all chosen attributes, for each scale $\epsilon$, and for different removal percentages. The maximum average similar similarity ratio, across the chosen values of $\epsilon$, is higher than $0.8$, which corresponds to \methodacronym being more than 4 times closer, in terms of AUC scores, to $\theta^\star_{-\mathcal{S}}$ than to $\theta^\star$. Furthermore, we observe that the region for $\epsilon$ in which \methodacronym is more similar to $\theta^\star_{-\mathcal{S}}$ than to $\theta^\star$ is consistent across different attributes and removal percentages. Increasing $\epsilon$ too much past this value causes a decrease in the similarity to $\theta^\star_{-\mathcal{S}}$, which corresponds to a degradation in the performance of $\hat{\theta}_{\epsilon}$.

Although the similarity ratio $\gamma_{\mathcal{S}}(\hat{\theta}_\epsilon; \theta^\star, \theta^\star_{-\mathcal{S}})$ gives a quantitative measure for the similarity in AUC scores between $\theta^\star_{-\mathcal{S}}$ and $\hat{\theta}_{\epsilon}$, we evaluate the consistency w.r.t. other performance measures, such as train loss, gradient norm or accuracy on all four dataset splits. In Table~\ref{table:results_celeba} we present the results on two randomly chosen attributes, removed fully or partially (50\%). The \methodacronym update maintains a similar loss to the two reference models, and a small gradient.
The average accuracies for $\theta^\star_{-\mathcal{S}}$ and $\hat{\theta}$ match closely, while the performances on $\mathcal{D}\setminus \mathcal{S}$ and $\mathcal{T}\setminus\mathcal{T}_a$ are almost unchanged. For partial removal, choosing the $\epsilon$ based on the similarity ratio can result in a model that mis-classifies more aggressively the removed attribute, compared to $\theta^\star_{-\mathcal{S}}$. This behavior is also influenced by the choice of the dampening $\lambda$ and suggests a more thorough hyper-parameter tuning is required.

As the problem is convex, we can consider the Euclidean distance $\|\hat{\theta}_\epsilon - \theta^\star_{-\mathcal{S}}\|$ between the parameters of \methodacronym and the parameters of the retrained from scratch model. As can be observed in Figure~\ref{fig:dists_gap_summary}, in general the region with the highest similarity ratio with respect to $\epsilon$ is close to the one where the minimum of the normalized distance $\frac{\| \hat{\theta}_{\epsilon} - \theta^\star_{-\mathcal{S}}\|}{\|\hat{\theta}_{\epsilon} -\theta^\star\| + \| \hat{\theta}_{\epsilon} - \theta^\star_{-\mathcal{S}}\|}$ is attained. However, the peak in similarity ratio does not typically match the minimum of the normalized parameter distance, the latter being attained with slightly higher values of $\epsilon$. It is worth noting that since this is a high dimensional problem, the Euclidean distance in parameter space might not be very informative. For instance, in the case of full removal of an attribute, the normalized Euclidean distance has a value close to $0.5$, which corresponds to \methodacronym being equally close to both $\theta^\star_{-\mathcal{S}}$ and $\theta^\star$; however, Table~\ref{table:results_celeba} shows that in fact \methodacronym has a more similar behavior to $\theta^\star_{-\mathcal{S}}$.  
Although we cannot conclude that $\hat{\theta}_\epsilon$ converges exactly to $\theta^\star_{-\mathcal{S}}$, nonetheless \methodacronym yields a model that is similar in behavior to retraining from scratch.

\paragraph{Comparison with the influence function update.} We note that in this particular case of multi-attribute convex classification the inverse of the Hessian matrix can be computed in a similar way, by using rank-one updates based on the Sherman-Morrison lemma. As the theory suggests, the influence function update does not require an additional scale hyper-parameter. When compared to \methodacronym in the same setup of (partially) removing attributes, we observe that the second-order update using the Hessian $H_{\mathcal{D}}(\theta^\star)$ has slightly higher similarity ratio (close to 0.9 across different removal percentages), compared to \methodacronym. However, we noticed that in the case of full removal of an attribute, the influence function update still correctly predicts some of the removed attributes, whereas the \methodacronym has more flexibility in fully removing the attribute (in terms of positive predictions), due to the additional scaling hyper-parameter. Furthermore, we note that the maximum similarity ratio for \methodacronym is attained for values of $\epsilon$ close to 1, which indicates that the empirical FIM and the Hessian have a similar scale in this particular case.

\subsection{Convex Multinomial Classification}

We test the behavior of the proposed \methodacronym update on a multi-classification task, for the Animals with Attributes 2 (AwA2) dataset \cite{XLSA18}, which contains 37322 images, from 50 classes of animals. We use a feature representation extracted from a ResNet101 \cite{he2016deep} architecture pre-trained on ImageNet \cite{imagenet}, as described in \cite{XLSA18}, with a 70\%/10\%/20\% train/validation/test split, and train an $\ell_2$-regularized linear model on top of the 2048-dimensional features, to predict each of the 50 classes. The linear classifier optimized using SGD with momentum, at a fixed learning rate, reaches a test accuracy of $93.08\%$.
Similarly to the multi-attribute setting, we compute and store the inverse of $F^{-1}_{\mathcal{D}}(\theta^\star)$ after training the initial model $\theta^\star$, and use it to perform the \methodacronym update. We choose the dampening equal to the $\ell_2$ regularization strength, namely $\lambda=10^{-4}$. Due to the large dimension of the parameter space, we assume a block-diagonal structure for $F^{-1}_{\mathcal{D}}(\theta^\star)$  and use blocks of size at most 45000. We perform a grid search for the scale $\epsilon$ and we use the confusion distance ratio defined in Equation~\ref{eqn:confusion_ratio} to select the optimal value of $\epsilon$. In all experiments we remove samples belonging to a single class, and for statistical significance we average the results over the first ten classes of the dataset. Furthermore, we consider both the scenarios in which the class is fully or only partially removed, at different percentages, ranging from 25\% to 100\%.

\begin{figure}[t]
\centering
\begin{subfigure}[t]{0.5\textwidth}
    \centering
    \includegraphics[height=1.5in]{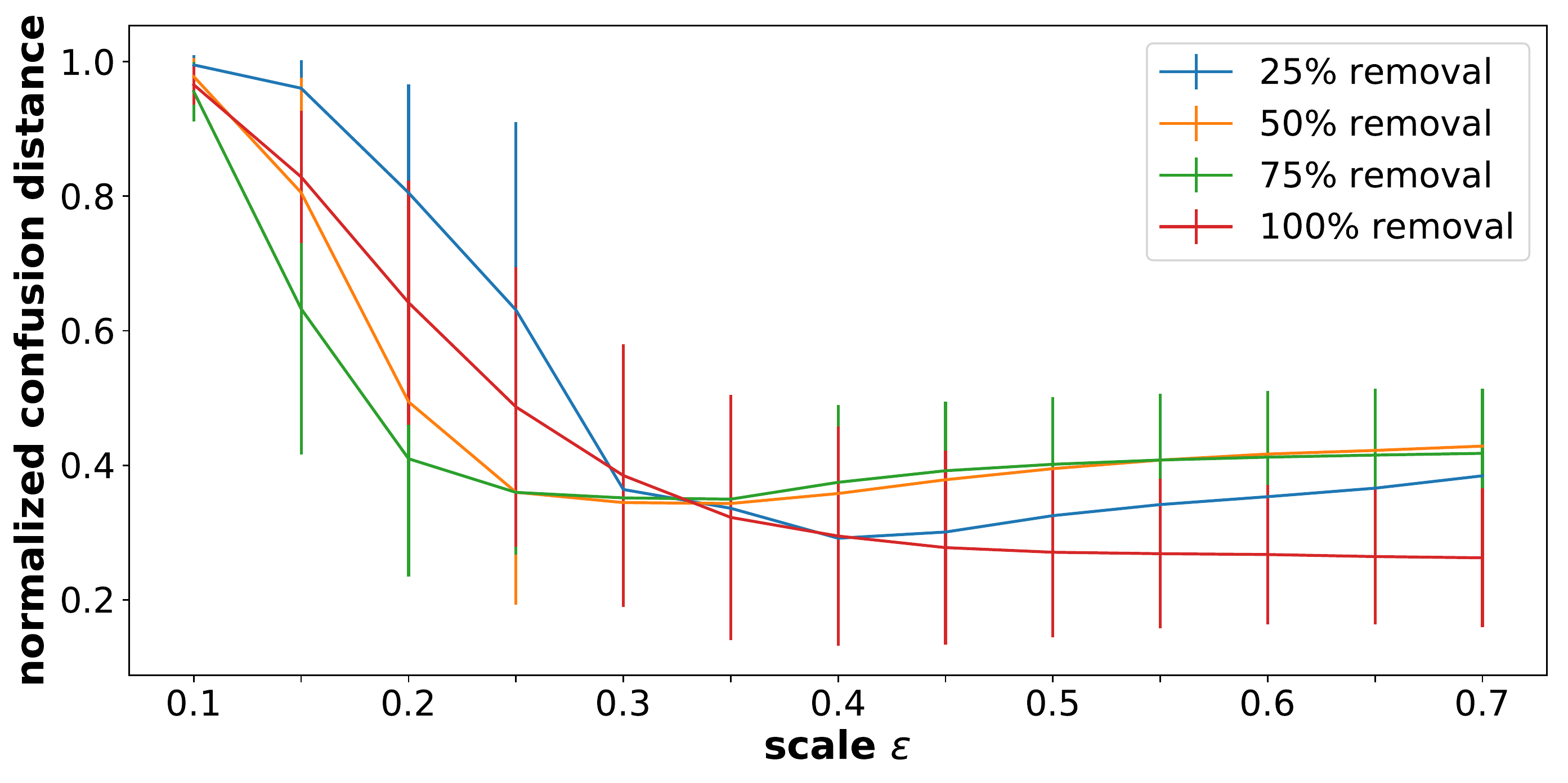}
    \caption{Normalized confusion distance on AwA2}
    \label{fig:confusion-awa2}
 \end{subfigure}%
 ~
 \begin{subfigure}[t]{0.5\textwidth}
    \centering
    \includegraphics[height=1.5in]{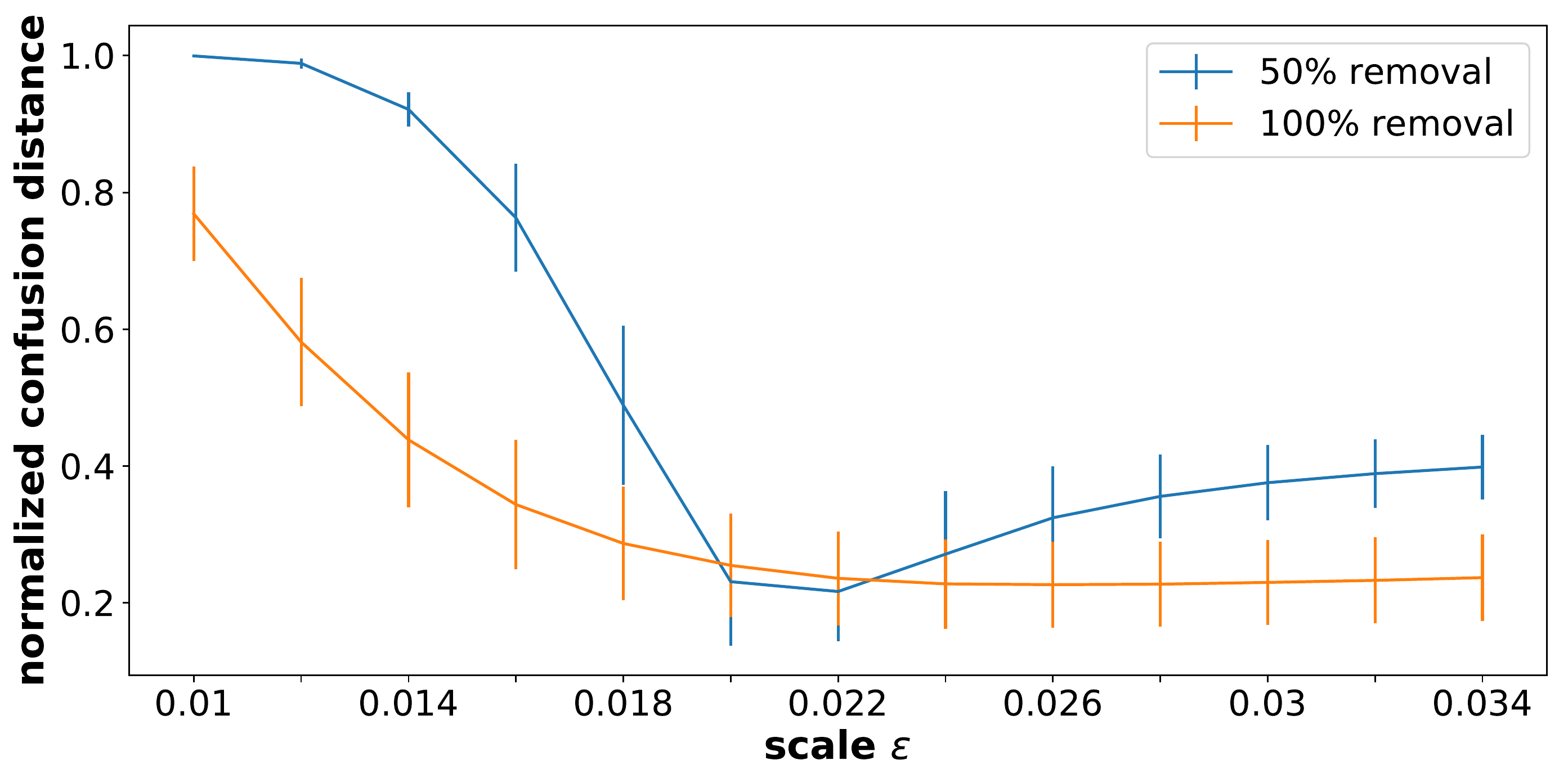}
    \caption{Normalized confusion distance on CIFAR10}
    \label{fig:confusion-cifar10}
 \end{subfigure}
 \caption{The normalized confusion distance. All samples removed in one step belong to a single class. We report the mean and standard deviation over the first 10 out of 50 classes for AwA2 and over the first 5 out of 10 classes for CIFAR10.}
 \label{fig:confusion-dist}
\end{figure}

\subsubsection{Results}
\label{sec:awa2-results}

\begin{figure*}[t]
\begin{center}
   \includegraphics[height=1.5in]{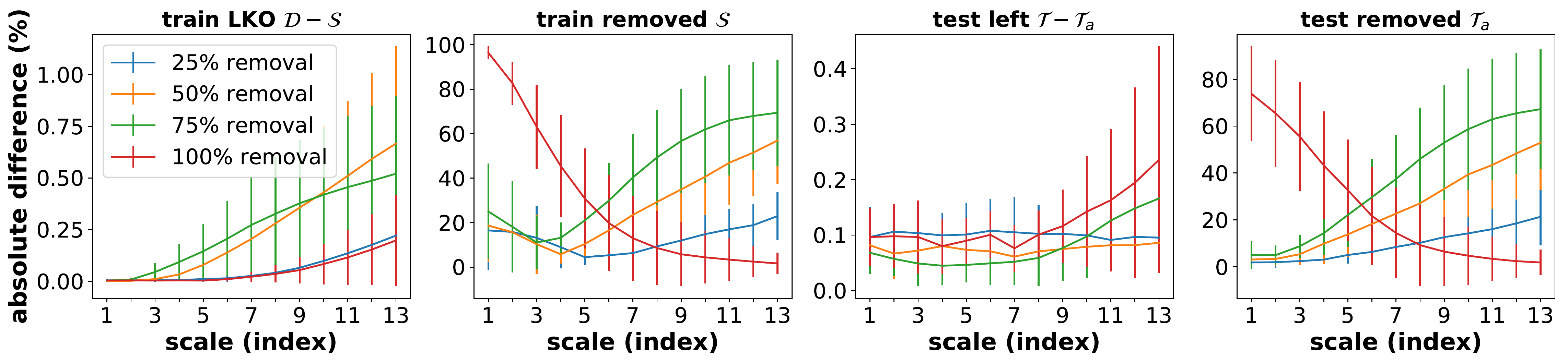}
\end{center}
   \caption{\textbf{(AwA2)} The absolute difference in accuracy between \methodacronym update $\hat{\theta}_\epsilon$ and $\theta^\star_{-\mathcal{S}}$, across all different data splits. Results are averaged across the first 10 different classes, for each considered removal percentage.} 
\label{fig:acc_diff_awa2}
\end{figure*}

Our main findings for the performance of the \methodacronym update on the AwA2 dataset are summarized in Figures~\ref{fig:confusion-awa2} and~\ref{fig:acc_diff_awa2}.
In Figure~\ref{fig:confusion-awa2} we plot the normalized confusion distance averaged across the first 10 classes from AwA2, as a function of the scale $\epsilon$, at different removal percentages. We identify a region for $\epsilon$ in which the confusion distance ratio decreases below $0.5$, which indicates that \methodacronym is closer to $\theta^\star_{-\mathcal{S}}$ than to $\theta^\star$, in terms of confusion matrices on $\mathcal{S}$. We can observe that for larger values of $\epsilon$, the confusion distance ratio plateaus or slowly converges towards a value close to $0.5$. The convergence towards $0.5$ indicates that the \methodacronym diverges away from both $\theta^\star$ and $\theta^\star_{-\mathcal{S}}$. In the case of full removal, the plateau is caused by the high dissimilarity in the confusion matrices of $\hat{\theta}_\epsilon$ and $\theta^\star$, which dominates the ratio. Therefore, a reasonable $\epsilon$ in case of full removal is one close to the start of the plateau. In general, a confusion similarity ratio close to $0.5$ indicates that either the predictions on the particular set coincide for $\theta^\star$ and $\theta^\star_{-\mathcal{S}}$, or that $\hat{\theta}_{\epsilon}$ is equally far from $\theta^\star$ and $\theta^\star_{-\mathcal{S}}$, which happens for example when the value of $\epsilon$ is too high and the model diverges. For partial removal, we can see that asymptotically $\delta_{\mathcal{S}}(\hat{\theta}_{\epsilon}; \theta^\star, \theta^\star_{-\mathcal{S}})$ approaches $0.5$ as the value of $\epsilon$ increases,
and we choose $\epsilon$ in the region corresponding to the minimum of the curve. Since the behavior $\delta_{\mathcal{S}}(\hat{\theta}_{\epsilon}; \theta^\star, \theta^\star_{-\mathcal{S}})$ gives no guarantees for the quality of the model on $\mathcal{D}\setminus\mathcal{S}$ or $\mathcal{T}$, we check its consistency w.r.t. more traditional measures of model quality, e.g. accuracy. 

Figure~\ref{fig:acc_diff_awa2} shows the absolute differences in accuracy between $\hat{\theta}_\epsilon$ and $\theta^\star_{-\mathcal{S}}$, computed on the four train and test splits. The results are averaged over the first 10 classes, for each removal percentage. The split on the test set in two different parts reveals a clearer picture of how much information of the removed samples is still preserved in the updated model. %
For each split there exists a region of $\epsilon$ in which the absolute differences in accuracy are small across all splits. A closer look at Figure~\ref{fig:confusion-awa2} shows that the values for $\epsilon$ corresponding to the minimum values of the normalized confusion distance are consistent with \methodacronym updates that result in the smallest difference in performance, compared to $\theta^\star_{-\mathcal{S}}$. However, it is worth noting that \methodacronym removes samples too aggressively in case of partial removal and that the regimes for $\epsilon$ where the \methodacronym is close to retraining from scratch are different in case of full and partial removal of a class.

\begin{figure*}
\begin{center}
    \includegraphics[height=1.5in]{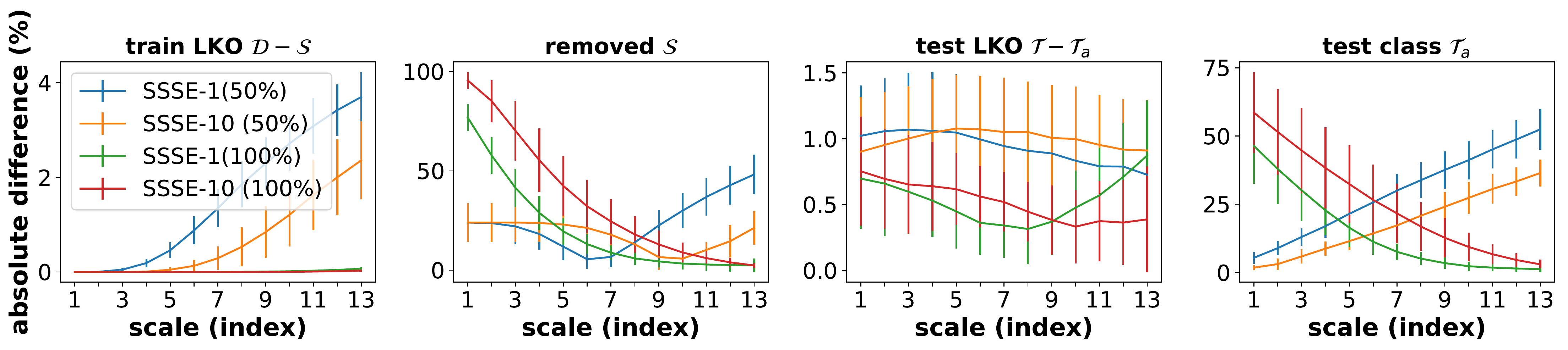}
\end{center}
\caption{\textbf{(CIFAR10)} Average absolute difference in accuracy between $\hat{\theta}_\epsilon$ and $\theta^\star_{-\mathcal{S}}$, when using different batch sizes for $\hat{F}^{-1}_{\mathcal{D}}(\theta^\star)$ in \methodacronym. Using mini-batches of 10 gradients does not have a major negative impact on the accuracy of the resulting model. Average and standard deviations over the first 5 classes are reported.}
\label{fig:ssse-batch}
\end{figure*}

\begin{figure*}[t]
\begin{center}
    \includegraphics[height=1.5in]{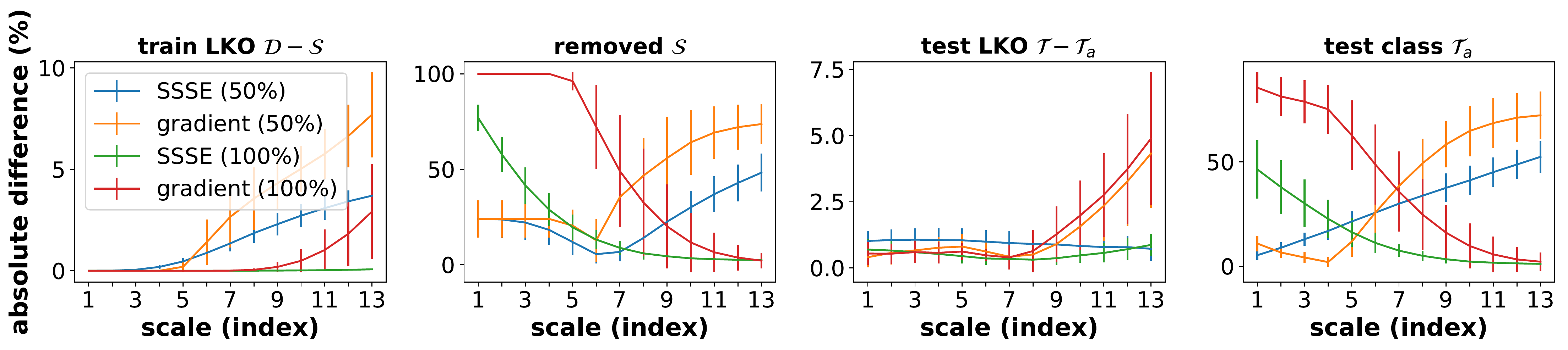}
\end{center}
\caption{\textbf{(CIFAR10)} Average absolute difference in accuracy between $\hat{\theta}_\epsilon$ and $\theta^\star_{-\mathcal{S}}$. We compare \methodacronym and a single GA step, when removing samples from the same class. The mean and standard deviation of the absolute differences in accuracy compared to $\theta^\star_{-\mathcal{S}}$ are computed for the first 5 classes.}
\label{fig:ssse-grad}
\end{figure*}

\subsection{Non-Convex Models}

Although we focused so far on convex models, \methodacronym is equally applicable in the non-convex setting. To illustrate this, we consider a ResNet20 \cite{he2016deep} architecture, trained on the CIFAR10 \cite{krizhevsky2009learning} dataset, with a 90\%--10\% train-validation split, optimized without additional data augmentation, using SGD with momentum, learning rate annealing and weight decay. We assume a block diagonal structure for the empirical FIM, and use blocks of size 70000. The value of the dampening parameter $\lambda$ was set to $0.0001$, which is different from the value of the weight decay used during training ($0.0005$). Similarly to the convex setting, we compute and store $F^{-1}_{\mathcal{D}}(\theta^\star)$ and use it to apply the \methodacronym update for each removal. We experiment with removing either all samples belonging to a single class, or a chosen percentage. 

Erasing samples from trained deep learning models is particularly difficult, as they have the capacity to effectively memorize the training set \cite{zhang2016understanding}. However, the \methodacronym update can still result in similar performance to a model trained from scratch on $\mathcal{D}\setminus \mathcal{S}$, from the same random seed as $\theta^\star$. In Figure~\ref{fig:ssse-batch} we show the absolute differences between the accuracy of \methodacronym and of $\theta^\star_{-\mathcal{S}}$, computed on the train and test splits of the available dataset. The results, as a function of the scale $\epsilon$, are averaged over the first five classes of CIFAR10, and for each class we consider full or partial (50\%) removal. For completeness, we also show the normalized confusion distance for full and partial (50\%) removal in Figure~\ref{fig:confusion-cifar10}, which confirms that the value of $\epsilon$ corresponding to the minimum average confusion distance ratio results in \methodacronym updates with the closest behavior in terms of accuracy on the different data splits (please see Figure~\ref{fig:ssse-batch}). To improve the efficiency of \methodacronym, we also use an approximation of the FIM, in which we take mini-batches of 10 gradients, instead of single sample gradients. Such approximations have been also recently used for neural network compression \cite{singh2020woodfisher}. With the appropriate scale $\epsilon$, both methods can achieve similar performance. For completeness, we also included in Figure~\ref{fig:ssse-grad} a comparison between \methodacronym and a single gradient ascent (GA) step on the removed samples; with an appropriate learning rate, GA can result in reasonable models, but not of the same quality as \methodacronym. 

We note that the task of erasing all samples from a given class is in fact easier than removing only a subset; the update can focus on perturbing, for example, mostly the weights in the final classification layer. For partial removal, the value of $\epsilon$ is more sensitive, and even in the ``optimal'' region we noticed that the \methodacronym tends to mis-classify the removed class samples more aggressively than $\theta^\star_{-\mathcal{S}}$. A similar effect was also noticed in the experiments with convex models. 
One explanation for this phenomenon would be related to the use of the gradient only on the removed samples $\mathcal{S}$ inside the \methodacronym update. Since the samples to be removed are selected randomly within a class, their gradient could be seen as an unbiased estimator for the gradient on all samples from that class. This, together with the fact that the FIM is computed on the entire training set, suggest that the update is oblivious to the fact that there are still some samples left in the training set that belong to the chosen class. Thus, \methodacronym would in this case have a similar behavior as to when all samples from the class are removed. We noticed that using inside \methodacronym the gradient on the leave-k-out training set $\mathcal{D}\setminus \mathcal{S}$ instead of only on $\mathcal{S}$ had a slight positive effect. Namely, we observed a decrease in the absolute difference in accuracy between $\hat{\theta}_\epsilon$ and $\theta^\star_{-\mathcal{S}}$ on both the removed samples $\mathcal{S}$ and the test samples belonging to the erased class $\mathcal{T}_a$. Therefore, we believe that improvements could be made to \methodacronym, to induce more aggressive perturbations in the higher layers. 

Finally, we note that the method closest to ours is \cite{golatkar2020eternal}, where the authors propose adding a noise term and instead use only the diagonal of the Fisher matrix. In  Figure~\ref{fig:ssse_soatto} we compare the results of our update step to the method from~\cite{golatkar2020eternal}. Contrary to what was reported in the original work, for us the noise-based update was not able to achieve good sample erasure at the same time as preserving the original model's accuracy. One potential explanation would be that the off-diagonal terms of the Fisher matrix carry significant information which cannot be captured only by using the diagonal elements.

\begin{figure}[t]
\begin{center}
   \includegraphics[height=1.45in]{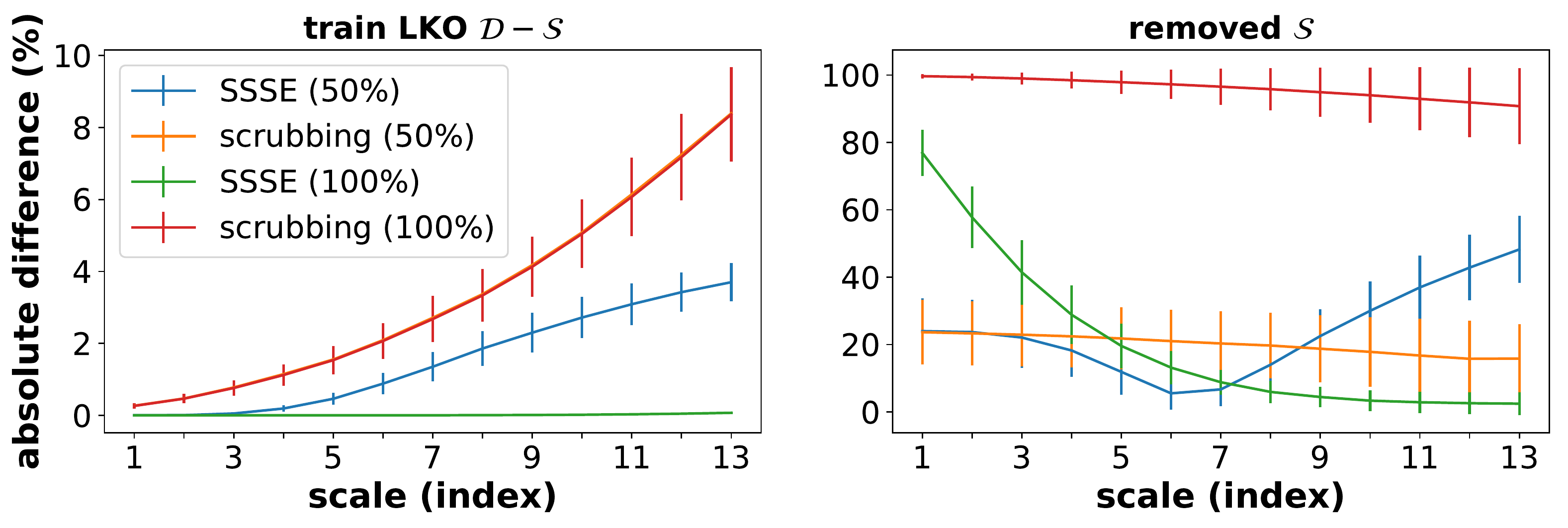}
\end{center}
   \caption{\textbf{(CIFAR10)} Comparison between \methodacronym and the scrubbing update from \cite{golatkar2020eternal}. We show the absolute differences in accuracy between each update and $\theta^\star_{-\mathcal{S}}$, on $\mathcal{D}\setminus\mathcal{S}$ and $\mathcal{S}$, for the task of removing samples from the same class, fully or partially (50\%). We repeat the experiment for the first 5 classes from CIFAR10 and report average and standard deviation of the results.}
\label{fig:ssse_soatto}
\end{figure}

\section{Discussion}

In this work we proposed \methodacronym, a method for erasing samples from a trained model. 
It is inspired by the concept of influence functions and made efficient through the use of the Fisher information matrix in combination with efficient low-rank matrix updates instead of an intractable Hessian matrix. 
Subsequently, we proposed two new similarity measures to evaluate the performance of our method compared to the gold standard of retraining from scratch, and we presented results for two convex classification problems, as well as an extension to the non-convex case. 

Our results demonstrate that influence-based model updates are not just theoretically a good idea for samples erasure, but that, with the right numerical tools, they can actually be made practical. 
We hope that this insight will inspire other researchers to build on our work and practitioners to add influence-based sample erasure to their toolboxes. 

\section*{Acknowledgements}
This project has received funding from the European Research Council (ERC) under the European Union's Horizon 2020 research and innovation programme (grant agreement No 805223 ScaleML).

{\small
\bibliographystyle{alpha}
\bibliography{papers}
}

\end{document}